\documentclass[letterpaper]{article} %
\usepackage[]{aaai23}  %
\usepackage{times}  %
\usepackage{helvet}  %
\usepackage{courier}  %
\usepackage[hyphens]{url}  %
\usepackage{graphicx} %
\usepackage{natbib}  %
\usepackage{caption} %
\usepackage{enumitem}

\usepackage{macros}

\title{Fully Dynamic Online Selection through Online Contention Resolution Schemes}
\author{
	Vashist Avadhanula\thanks{Research performed while the author was working at Meta.}, Andrea Celli\textsuperscript{\rm 1}, Riccardo Colini-Baldeschi\textsuperscript{\rm 2},\\ Stefano Leonardi\textsuperscript{\rm 3}, Matteo Russo\textsuperscript{\rm 3}\\
}
\affiliations{
    \textsuperscript{\rm 1}Department of Computing Sciences, Bocconi University, Milan, Italy\\
    \textsuperscript{\rm 2} Core Data Science, Meta, London, UK\\
    \textsuperscript{\rm 3}Department of Computer, Control and Management Engineering, Sapienza University, Rome, Italy\\
    vas1089@gmail.com, andrea.celli2@unibocconi.it, rickuz@fb.com, leonardi@diag.uniroma1.it, mrusso@diag.uniroma1.it
}

\begin{document}

\maketitle

\begin{abstract}
We study \emph{fully dynamic} online selection problems in an adversarial/stochastic setting that includes Bayesian online selection, prophet inequalities, posted price mechanisms, and stochastic probing problems subject to combinatorial constraints.  
In the classical ``incremental'' version of the problem, selected elements remain active until the end of the input sequence. On the other hand, in the \emph{fully dynamic} version of the problem, elements stay active for a limited time interval, and then leave. This models, for example, the online matching of tasks to workers with task/worker-dependent working times, and sequential posted pricing of perishable goods.
A successful approach to online selection problems in the adversarial setting is given by the notion of \emph{Online Contention Resolution Scheme} (OCRS), that uses \emph{a priori} information to formulate a linear relaxation of the underlying optimization problem, whose optimal fractional solution is rounded online for any adversarial order of the input sequence. Our main contribution is providing a general method for constructing an OCRS for fully dynamic online  selection problems. Then, we show how to employ such OCRS to construct no-regret algorithms in a partial information model with \emph{semi-bandit} feedback and adversarial inputs. 
\end{abstract}

\section{Introduction}\label{sec:intro}

Consider the case where a financial service provider receives multiple operations every hour/day. These operations might be malicious. The provider needs to assign them to \emph{human reviewers} for inspection. The time required by each reviewer to file a reviewing task and the reward (weight) that is obtained with the review follow some distributions. The distributions can be estimated from historical data, as they depend on the type of transaction that needs to be examined and on the expertise of the employed reviewers. To efficiently solve the problem, the platform needs to compute a matching between tasks and reviewers based on the a priori information that is available. However, the time needed for a specific review, and the realized reward (weight), is often known only after the task/reviewer matching is decided. 

A multitude of variations to this setting are possible. For instance, if a cost is associated with each reviewing task, the total cost for the reviewing process might be bounded by a budget. Moreover, there might be various kinds of restrictions on the subset of reviewers that are assigned at each time step. Finally, the objective function might not only be the sum of the rewards (weights) we observe, if, for example, the decision maker has a utility function with ``diminishing return'' property.

To model the general class of sequential decision problems described above, we introduce \emph{fully dynamic online selection problems}. This model generalizes online selection problems \cite{CVZ11}, where elements arrive online in an adversarial order and algorithms can use a priori information to maximize the weight of the selected subset of elements, subject to combinatorial constraints (such as matroid, matching, or knapsack). 

In the classical version of the problem \cite{CVZ11}, once an element is selected, it will affect the combinatorial constraints throughout the entire input sequence.
This is in sharp contrast with the fully dynamic version, where an element will affect the combinatorial constraint only for a limited time interval, which we name \emph{activity time} of the element. For example, a new task can be matched to a reviewer as soon as she is done with previously assigned tasks, or an agent can buy a new good as soon as the previously bought goods are perished. A large class of Bayesian online selection \cite{KleinbergW12}, prophet inequality \cite{HKS07}, posted price mechanism \cite{CHMS10}, and stochastic probing \cite{GN13} problems that have been studied in the classical version of online selection can therefore be extended to the fully dynamic setting. Note that in the dynamic algorithms literature, \emph{fully dynamic} algorithms are algorithms that deal with both adversarial insertions and deletions \citep{DEGI10}. We could also interpret our model in a similar sense since elements arrive online (are inserted) according to an adversarial order, and cease to exist (are deleted) according to adversarially established activity times.

A successful approach to online selection problems is based on Online Contention Resolution Schemes (OCRSs) \cite{FSZ16}. OCRSs use a priori information on the values of the elements to formulate a linear relaxation whose optimal fractional solution upper bounds the performance of the integral offline optimum. Then, an online rounding procedure is used to produce a solution whose value is as close as possible to the fractional relaxation solution's value, for any adversarial order of the input sequence. The OCRS approach allows to obtain good approximations of the expected optimal solution for linear and submodular objective functions. The existence of OCRSs for fully dynamic online selection problems is therefore a natural research question that we address in this work.

The OCRS approach is based on the availability of a priori information on weights and activity times.  However, in real world scenarios, these might be missing or might be expensive to collect. Therefore, in the second part of our work, we study the fully dynamic online selection problem with partial information, where the main research question is whether the OCRS approach is still viable if a priori information on the weights is missing. 
In order to answer this question, we study a repeated version of the fully dynamic online selection problem, in which at each stage weights are unknown to the decision maker (i.e., no a priori information on weights is available) and chosen adversarially.
The goal in this setting is the design of an online algorithm with performances (i.e., cumulative sum of weights of selected elements) close to that of the best static selection strategy in hindsight.

\subsection{Our Contributions}

First, we introduce the \emph{fully dynamic online selection problem}, in which elements arrive following an adversarial ordering, and revealed one-by-one their weights and activity times at the time of arrival (i.e., \emph{prophet model}), or after the element has been selected (i.e., \emph{probing model}). Our model describes temporal packing constraints (i.e., downward-closed), where elements are active only within their activity time interval. The objective is to maximize the weight of the selected set of elements subject to temporal packing constraints.
We provide two black-box reductions for adapting classical OCRS for online (non-dynamic) selection problems to the fully dynamic setting under full and partial information. 
\begin{itemize}[leftmargin=-1pt]
\item[] \textbf{Blackbox reduction 1: from OCRS to temporal OCRS.} Starting from a $(b,c)$-selectable greedy OCRS in the classical setting, we use it as a subroutine to build a $(b,c)$-selectable greedy OCRS in the more general temporal setting (see \Cref{alg:spattemp} and \Cref{thm:red}). This means that competitive ratio guarantees in one setting determine the same guarantees in the other. Such a reduction implies the existence of algorithms with constant competitive ratio for online optimization problems with linear or submodular objective functions subject to matroid, matching, and knapsack constraints, for which we give explicit constructions. We also extend the framework to elements arriving in batches, which can have correlated weights or activity times within the batch, as described in the appendix of the paper.

\item[] \textbf{Blackbox reduction 2: from temporal OCRS to no-$\alpha$-regret algorithm.} Following the recent work by \citet{gergatsouli2022online} in the context of Pandora's box problems, we define the following extension of the problem to the partial-information setting. For each of the $T$ stages, the algorithm is given in input a new instance of the fully dynamic online selection problem. Activity times are fixed beforehand and known to the algorithm, while weights are chosen by an adversary, and revealed only after the selection at the current stage has been completed. 
In such setting, we show that an $\alpha$-competitive temporal OCRS can be exploited in the adversarial partial-information version of the problem, in order to build no-$\alpha$-regret algorithms with polynomial per-iteration running time. Regret is measured with respect to the cumulative weights collected by the best fixed selection policy in hindsight. We study three different settings: in the first setting, we study the \emph{full-feedback model} (i.e., the algorithm observes the entire utility function at the end of each stage). Then, we focus on the \emph{semi-bandit-feedback model}, in which the algorithm only receives information on the weights of the elements it selects. In such setting, we provide a no-$\alpha$-regret framework with $\tilde O(T^{1/2})$ upper bound on  cumulative regret in the case in which we have a ``white-box'' OCRS (i.e., we know the exact procedure run within the OCRS, and we are able to simulate it ex-post). Moreover, we also provide a no-$\alpha$-regret algorithm with $\tilde O(T^{2/3})$ regret upper bound for the case in which we only have oracle access to the OCRS (i.e., the OCRS is treated as a black-box, and the algorithm does not require knowledge about its internal procedures).

\end{itemize}

\subsection{Related Work}

In the first part of the paper, we deal with a setting where the  algorithm has complete information over the input but is unaware of the order in which elements arrive. In this context, Contention resolution schemes (CRS) were introduced by~\citet{CVZ11} as a powerful rounding technique in the context of submodular maximization. The CRS framework was extended to online contention resolution schemes (OCRS) for online selection problems by \citet{FSZ16}, who provided
constant competitive OCRSs for different problems, e.g. intersections of matroids, matchings, and prophet inequalities. We generalize the OCRS framework to a setting where elements are timed and cease to exist right after.

In the second part, we lift the complete knowledge assumption and work in an adversarial bandit setting, where at each stage the entire set of elements arrives, and we seek to select the ``best'' feasible subset. This is similar to the problem of \emph{combinatorial bandits} \cite{cesa2012combinatorial}, but unlike it, we aim to deal with combinatorial selection of \emph{timed} elements. In this respect, \emph{blocking bandits} \cite{BSSS19} model situations where played arms are blocked for a specific number of stages. Despite their contextual \cite{BPCS21}, combinatorial \cite{AtsidakouP0CS21}, and adversarial \cite{BishopCMT20} extensions, recent work on blocking bandits only addresses specific cases of the fully dynamic online selection problem \cite{dickerson2018allocation}, which we solve in entire generality, i.e. adversarially and for all packing constraints.

Our problem is also related to \emph{sleeping bandits} \cite{KNMS10}, in that the adversary decides which actions the algorithm can perform at each stage $t$. Nonetheless, a sleeping bandit adversary has to communicate all available actions to the algorithm before a stage starts, whereas our adversary sets arbitrary activity times for each element, choosing in what order elements arrive.

\section{Preliminaries}\label{sec:prel}

Given a finite set $\cX\subseteq\mathbb{R}^n$ and $\cY\subseteq 2^{\cX}$, let $\bone_{\cY}\in\{0,1\}^{|X|}$ be the characteristic vector of set $\cX$, and $\co{\cX}$ be the convex hull of $\cX$. We denote vectors by bold fonts. Given vector $\bx$, we denote by $x_i$ its $i$-th component. The set $\{1,2,\ldots,n\}$, with $n\in\mathbb{N}_{>0}$, is compactly denoted as $[n]$. Given a set $\cX$ and a scalar $\alpha\in\mathbb{R}$, let $\alpha\cX\defeq\mleft\{\alpha x: x\in\cX\mright\}$. Finally, given a discrete set $\cX$, we denote by $\Delta^{\cX}$ the $|\cX|$-simplex.

We start by introducing a general \emph{selection problem} in the standard (i.e., non-dynamic) case as studied by \citet{KleinbergW12} in the context of prophet inequalities. Let $\cE$ be the ground set and let $m\defeq |\cE|$.
Each element $e\in\cE$ is characterized by a collection of parameters $z_e$. In general, $z_e$ is a random variable drawn according to an element-specific distribution $\cZ_e$, supported over the joint set of possible parameters. In the standard (i.e., non-dynamic) setting,  $z_e$ just encodes the \emph{weight} associated to element $e$, that is  $z_e=(w_e)$, for some $w_e\in[0,1]$.\footnote{This is for notational convenience. In the dynamic case $z_e$ will contain other parameters in addition to weights.} In such case distributions $\cZ_e$ are supported over $[0,1]$. Random variables $\mleft\{z_e: e\in\cE\mright\}$ are independent, and $z_e$ is distributed according to $\cZ_e$.
An \emph{input sequence} is an ordered sequence of elements and weights such that every element in $\cE$ occurs exactly once in the sequence. The order is specified by an arrival time $s_e$ for each element $e$. Arrival times are such that $s_e\in [m]$ for all $e\in\cE$, and for two distinct $e, e'$ we have $s_e \neq s_{e'}$.  The order of arrival of the elements is a priori unknown to the algorithm, and can be selected by an adversary. In the standard full-information setting the distributions $\cZ_e$ can be chosen by an adversary, but they are known to the algorithm a priori.
We consider problems characterized by a family of \emph{packing constraints}. 
\begin{definition}[Packing Constraint]\label{def:packing}
	A family of constraints $\cF = (\cE, \cI)$, for ground set $\cE$ and independence family $\cI \subseteq 2^{\cE}$, is said to be \emph{packing} (i.e., \emph{downward-closed}) if, taken $A \in \cI$, and $B \subseteq A$, then $B \in \cI$.
\end{definition}

\noindent Elements of  $\cI$ are called \emph{independent sets}. Such family of constraints is closed under intersection, and encompasses matroid, knapsack, and matching constraints. 

\paragraph{Fractional LP formulation}
Even in the offline setting, in which the ordering of the input sequence $(s_e)_{e\in\cE}$ is known beforehand, determining an independent set of maximum cumulative weight may be \textsc{NP}-hard in the worst-case \citep{Fei98}. 
Then, we consider the relaxation of the problem in which we look for an optimal \emph{fractional solution}. The value of such solution is an upper bound to the value of the true offline optimum. Therefore, any algorithm guaranteeing a constant approximation to the offline fractional optimum immediately yields the same guarantees with respect to the offline optimum. 
Given a family of packing constraints $\cF=(\cE,\cI)$, in order to formulate the problem of computing the best fractional solution as a linear programming problem (LP) we introduce the notion of \emph{packing constraint polytope} $\cP_\cF \subseteq [0,1]^m$ which is such that
	$
	\cP_\cF := \co\left(\{\bone_{S} : S \in \cI\}\right).
	$
\noindent Given a non-negative submodular function $f: [0,1]^m \rightarrow \mathbb{R}_{\ge0}$, and a family of packing constraints $\cF$, an optimal \emph{fractional} solution can be computed via the LP $\max_{\bx\in\cP_\cF} f(\bx)$. If the goal is maximizing the cumulative sum of weights, the objective of the optimization problem is $\langle \bx,\bw\rangle$, where $\bw\defeq (w_1,\ldots,w_m)\in[0,1]^m$ is a vector specifying the weight of each element.
If we assume access to a polynomial-time separation oracle for $\cP_\cF$ such LP yields an optimal fractional solution in polynomial time. 

\begin{itemize}[leftmargin=-1pt]
\item[] \textbf{Online selection problem.} In the online version of the problem, given a family of packing constraints $\cF$, the goal is selecting an independent set whose cumulative weight is as large as possible. In such setting, the elements reveal one by one their realized $z_e$, following a fixed prespecified order unknown to the algorithm. Each time an element reveals $z_e$, the algorithm has to choose whether to select it or discard it, before the next element is revealed. Such decision is irrevocable. Computing the exact optimal solution to such online selection problems is intractable in general \citep{Fei98}, and the goal is usually to design approximation algorithms with good \emph{competitive ratio}.\footnote{The \emph{competitive ratio} is computed as the worst-case ratio between the value of the solution found by the algorithm and the value of an optimal solution.} In the remainder of the section we describe one well-known framework for such objective.
\item[] \textbf{Online contention resolution schemes.} \emph{Contention resolution schemes} were originally proposed by \citet{CVZ11} in the context of submodular function maximization, and later extended to online selection problems by \citet{FSZ16} under the name of \emph{online contention resolution schemes} (OCRS).
Given a fractional solution $\bx\in\cP_\cF$, an OCRS is an online rounding procedure yielding an independent set in $\cI$ guaranteeing a value close to that of $\bx$. 
Let $R(\bx)$ be a random set containing each element $e$ independently and with probability $x_e$. 
The set $R(\bx)$ may not be feasible according to constraints $\cF$. An OCRS essentially provides a procedure to construct a good feasible approximation by starting from the random set $R(\bx)$. Formally,
\end{itemize}

\begin{definition}[OCRS] Given a point $\bx\in\cP_\cF$ and the set of elements $R(\bx)$, elements $e\in\cE$ reveal one by one whether they belong to $R(\bx)$ or not. An OCRS chooses irrevocably whether to select an element in $R(\bx)$ before the next element is revealed. An OCRS for $\cP_\cF$ is an online algorithm that selects $S\subseteq R(\bx)$ such that $\bone_{S}\in\cP_\cF$.
\end{definition}

\noindent We will focus on \emph{greedy OCRS}, which were defined by \citet{FSZ16} as follows. 

\begin{definition}[Greedy OCRS]\label{def:OCRS}
	Let $\cP_\cF \subseteq [0,1]^m$ be the feasibility polytope for constraint family $\cF$. An OCRS $\pi$ for $\cP_\cF$ is called a greedy OCRS if, for every ex-ante feasible solution $\bx \in \cP_\cF$, it defines a packing subfamily of feasible sets $\cF_{\pi, \bx} \subseteq \cF$, and an element $e$ is selected upon arrival if, together with the set of already selected elements, the resulting set is in $\cF_{\pi, \bx}$. 
\end{definition}
\noindent A greedy OCRS is \emph{randomized} if, given $\bx$, the choice of $\cF_{\pi, \bx}$ is randomized, and \emph{deterministic} otherwise.
For $b,c\in[0,1]$, we say that a greedy OCRS $\pi$ is \emph{$(b,c)$-selectable} if, for each $e \in \cE$, and given $\bx\in b\cP_{\cF}$ (i.e., belonging to a down-scaled version of $\cP_\cF$),
\[
\Pr_{\pi, R(\bx)}\left[S \cup \{e\} \in \cF_{\pi, \bx} \quad \forall S \subseteq R(\bx), S \in \cF_{\pi, \bx}\right] \geq c.
\]
Intuitively, this means that, with probability at least $c$, the random set $R(\bx)$ is such that an element $e$ is selected no matter what other elements $I$ of $R(x)$ have been selected so far, as long as $I \in \cF_{\pi,\bx}$.
This guarantees that an element is selected with probability at least $c$ against any adversary, which implies a $bc$ competitive ratio with respect to the offline optimum (see Appendix \ref{app:ocrs} for further details).  
Now, we provide an example due to \citet{FSZ16} of a feasibility constraint family where OCRSs guarantee a constant competitive ratio against the offline optimum. We will build on this example throughout the paper in order to provide intuition for the main concepts. 

\begin{example}[Theorem 2.7 in \citep{FSZ16}]\label{ex:running}
Given a graph $G = (\cV, \cE)$, with $|\cE| = m$ edges, we consider a matching feasibility polytope $\cP_\cF = \left\{\bx \in [0, 1]^{m} :\sum_{e\in\delta(u)}x_e \leq 1 , \forall u \in \cV\right\}$, where $\delta(u)$ denotes the set of all adjacent edges to $u \in \cV$. Given $b\in[0,1]$, the OCRS takes as input $\bx \in b\cP_\cF$, and  samples each edge $e$ with probability $x_e$ to build $R(\bx)$. Then, it selects each edge $e\in R(\bx)$, upon its arrival, with probability $(1 - e^{-x_e})/x_e$ only if it is feasible. Then, the probability to select any edge $e = (u,v)$ (conditioned on being sampled) is 
\begin{align*}
    &\frac{1 - e^{-x_e}}{x_e} \cdot \prod_{e^{\prime} \in \delta(u) \cup \delta(v) \setminus \{e\}} e^{-x_{e^\prime}} \\
    & = \frac{1 - e^{-x_e}}{x_e} \cdot e^{-\sum_{e^{\prime} \in \delta(u) \cup \delta(v) \setminus \{e\}}x_{e^\prime}} \geq \frac{1 - e^{-x_e}}{x_e} \cdot e^{-2b} \\
    & \geq e^{-2b},
\end{align*}
where the inequality follows from $x_e \in b\cP_\cF$, i.e., $\sum_{e'\in\delta(u)\setminus \{e\}}x_{e'} \leq b - x_e$, and similarly for $\delta(v)$. Note that in order to obtain an unconditional probability, we need to multiply the above by a factor $x_e$.
\end{example}

We remark that this example resembles closely our introductory motivating application, where financial transactions need to be assigned to reviewers upon their arrival.
Moreover, \citet{FSZ16} give explicit constructions of $(b,c)$-selectable greedy OCRSs for knapsack, matching, matroidal constraints, and their intersection. We include a discussion of their feasibility polytopes in Appendix \ref{app:A}.
\citet{EFGT20} generalize the above online selection procedure to a setting where elements arrive in batches rather than one at a time; we provide a discussion of such setting in Appendix \ref{app:B}.
\section{Fully Dynamic Online Selection}\label{sec:temporal}

\noindent The \emph{fully dynamic online selection problem} is characterized by the definition of \emph{temporal packing constraints}. We generalize the online selection model (Section~\ref{sec:prel}) by introducing an \emph{activity time} $d_e \in [m]$ for each element. Element $e$ arrives at time $s_e$ and, if it is selected by the algorithm, it remains active up to time $s_e+d_e$ and ``blocks'' other elements from being selected. Elements arriving after that time can be selected by the algorithm. In this setting, each element $e\in\cE$ is characterized by a tuple of attributes $z_e := (w_e, d_e)$. Let $\cF^{\bd} \defeq (\cE, \cI^{\bd})$ be the family of \emph{temporal packing feasibility constraints} where elements block other elements in the same independent set according to activity time vector $\bd = (d_e)_{e \in \cE}$.
The goal of fully dynamic online selection is selecting an independent set in $\cI^{\bd}$ whose cumulative weight is as large as possible (i.e., as close as possible to the offline optimum). We can naturally extend the expression for packing polytopes in the standard setting to the temporal one for every feasibility constraint family, by exploiting the following notion of \emph{active elements}.

\begin{definition}[Active Elements]\label{def:active}
For element $e \in \cE$ and given $\{z_e\}_{e\in\cE}$, we denote the set of \emph{active elements} as $\cE_e \defeq \mleft\{e^\prime \in \cE : s_{e^\prime} \leq s_e \leq s_{e^\prime} + d_{e^\prime}\mright\}$.\footnote{Note that, since for distinct elements $e, e^\prime$, we have $s_{e^\prime} \neq s_e$, we can equivalently define the set of active elements as $\cE_e \defeq \left\{e^\prime \in \cE : s_{e^\prime} < s_e \leq s_{e^\prime} + d_{e^\prime}\right\} \cup \{e\}$.}
\end{definition}

\noindent In this setting, we don't need to select an independent set $S \in \cF$, but, in a less restrictive way, we only require that for each incoming element we select a feasible subset of the set of active elements.  
\begin{definition}[Temporal packing constraint polytope]\label{def:schedpackpol}
Given $\cF=(\cE,\cI)$, a \emph{temporal packing constraint polytope} $\cP^{\bd}_\cF \subseteq [0,1]^m$ is such that
$
    \cP^{\bd}_\cF \defeq \co\left(\{\bone_{S} : S \cap \cE_e \in \cI, \forall e \in \cE\}\right).
$
\end{definition}

\begin{observation}\label{obs:inf}
For a fixed element $e$, the temporal polytope is the convex hull of the collection containing all the sets such that $S \cap \cE_e$ is feasible. This needs to be true for all $e \in \cE$, meaning that we can rewrite the polytope and the feasibility set as $ \cP^{\bd}_\cF = \co\left(\bigcap_{e \in\cE}{\{\bone_{S} : S \cap \cE_e \in \cF\}}\right)$, and $\cI^{\bd} = \bigcap_{e \in \cE}{\{S : S \cap \cE_e \in \cI\}}$.
Moreover, when $\bd$ and $\bd^\prime$ differ for at least one element $e$, that is $d_e < d^\prime_e$, then $\cE_e \subseteq \cE^\prime_e$. Then, $\cP^{\bd}_\cF \supseteq \cP^{\bd^\prime}_\cF$, $\cI^{\bd} \supseteq \cI^{\bd^\prime}$.
\end{observation}

We now extend \Cref{ex:running} to account for activity times. In Appendix \ref{app:A} we also work out the reduction from standard to \emph{temporal} packing constraints for a number of examples, including rank-1 matroids (single-choice), knapsack, and general matroid constraints.

\begin{example}\label{ex:running2}
We consider the temporal extension of the matching polytope presented in \Cref{ex:running}, that is \[\cP^{\bd}_\cF = \left\{\by \in [0, 1]^{m} :\sum_{e\in\delta(u) \cap \cE_e}x_e \leq 1 , \forall u \in V, \forall e \in \cE\right\}.\] Let us use the same OCRS as in the previous example, but where ``feasibility'' only concerns the subset of active edges in $\delta(u) \cup \delta(v)$. The probability to select an edge $e = (u,v)$ is
\begin{align*}
    \frac{1 - e^{-x_e}}{x_e} \cdot \hspace{-0.2cm} \prod_{e^{\prime} \in \delta(u) \cup \delta(v) \cap \cE_e \setminus \{e\}} \hspace{-0.8cm} e^{-x_{e^\prime}}
    \geq \frac{1 - e^{-x_e}}{x_e} \cdot e^{-2b} \geq e^{-2b},
\end{align*}
which is obtained in a similar way to \Cref{ex:running}.
\end{example}

The above example suggests to look for a general reduction that maps an OCRS for the standard setting, to an OCRS for the temporal setting, while achieving at least the same competitive ratio.

\section{OCRS for Fully Dynamic Online Selection}\label{sec:temporal greedy ocrs}

The first black-box reduction which we provide consists in showing that a $(b,c)$-selectable greedy OCRS for standard packing constraints implies the existence of a $(b,c)$-selectable greedy OCRS for temporal constraints. 
In particular, we show that the original greedy OCRS working for $\bx \in b\cP_{\cF}$ can be used to construct another greedy OCRS for $\by \in b\cP^{\bd}_{\cF}$. 
To this end, Algorithm \ref{alg:spattemp} provides a way of exploiting the original OCRS $\pi$ in order to manage temporal constraints. 
For each element $e$, and given the induced subfamily of packing feasible sets $\cF_{\pi,\by}$, the algorithm checks whether the set of previously selected elements $S^{\bd}$ which are still active in time, together with the new element $e$, is feasible with respect to $\cF_{\pi,\by}$. If that is the case, the algorithm calls the OCRS $\pi$. Then, if the OCRS $\pi$ for input $\by$ decided to select the current element $e$, the algorithm adds it to $S^{\bd}$, otherwise the set remains unaltered.
We remark that such a procedure is agnostic to whether the original greedy OCRS is deterministic or randomized.
We observe that, due to a larger feasibility constraint family, the number of independent sets have increased with respect to the standard setting. However, we show that this does not constitute a problem, and an equivalence between the two settings can be established through the use of Algorithm \ref{alg:spattemp}.
The following result shows that Algorithm~\ref{alg:spattemp} yields a $(b,c)$-selectable greedy OCRS for temporal packing constraints.

\RestyleAlgo{ruled}

\begin{algorithm}[tb]
\caption{\textbf{Greedy OCRS Black-box Reduction}}\label{alg:spattemp}
\textbf{Input:} Feasibility families $\cF$ and $\cF^{\bd}$, polytopes $\cP_{\cF}$ and $\cP^{\bd}_{\cF}$, OCRS $\pi$ for $\cF$, a point $\bx \in b\cP^{\bd}_{\cF}$\;
Initialize $S^{\bd} \leftarrow \emptyset$\;
Sample $R(\bx)$ such that $\emph{\Pr}\left[e \in R(\bx)\right] = x_e$\;
\For{$e \in \cE$}
{
    Upon arrival of element $e$, compute the set of currently active elements $\cE_e$\;
    \eIf{$(S^{\bd} \cap \cE_e) \cup \{e\} \in \cF_{\pi, \by}$}
    {
        Execute the original greedy OCRS $\pi(\bx)$\;
        Update $S^{\bd}$ accordingly\;
    }{
        Discard element $e$\;
    }
}
\textbf{return} set $S^{\bd}$\;
\end{algorithm}

\begin{theorem}\label{thm:red}
Let $\cF, \cF^{\bd}$ be the standard and temporal packing constraint families, respectively, and let their corresponding polytopes be  $\cP_{\cF}$ and $\cP^{\bd}_{\cF}$. Let $\bx \in b\cP_{\cF}$ and $\by \in b\cP^{\bd}_{\cF}$, and consider a $(b,c)$-selectable greedy OCRS $\pi$ for $\cF_{\pi, \bx}$. Then, Algorithm~\ref{alg:spattemp} equippend with $\pi$ is a $(b,c)$-selectable greedy OCRS for $\cF^{\bd}_{\pi, \by}$.
\end{theorem}

\begin{proof}

Let us denote by $\hat\pi$ the procedure described in Algorithm \ref{alg:spattemp}. First, we show that $\hat\pi$ is a greedy OCRS for $\cF^{\bd}$.
\begin{itemize}[leftmargin=*]
\item[] \textbf{Greedyness.} It is clear from the setting and the construction that elements arrive one at a time, and that $\hat\pi$ irrevocably selects an incoming element only if it is feasible, and before seeing the next element. Indeed, in the \emph{if} statement of Algorithm \ref{alg:spattemp}, we check that the active subset of the elements selected so far, together with the new arriving element $e$, is feasible against the subfamily $\cF_{\pi, \bx} \subseteq \cF$. 
Constraint subfamily $\cF_{\pi, \bx}$ is induced by the original OCRS $\pi$, and  point $\bx$ belongs to the polytope $b\cP^{\bd}_{\cF}$. 
Note that we do not necessarily add element $e$ to the running set $S^{\bd}$, even though feasible, but act as the original greedy OCRS would have acted.
All that is left to be shown is that such a procedure defines a subfamily of feasibility constraints $\cF^{\bd}_{\pi, \bx} \subseteq \cF^{\bd}$. By construction, on the arrival of each element $e$, we guarantee that $S^{\bd}$ is a set such that its subset of active elements is feasible. This means that $S^{\bd} \cap \cE_e \in \cF_{\pi, \bx} \subseteq \cF$. Then,
\begin{align*}
    S^{\bd} \in \cF^{\bd}_{\pi, \bx} := \bigcap_{e \in \cE}{\{S : S \cap \cE_e \in \cF_{\pi, \bx}\}}.
\end{align*}
Finally, $\cF_{\pi, \bx} \subseteq \cF$ implies that $\cF^{\bd}_{\pi, \bx} \subseteq \cF^{\bd}$, which shows that $\hat\pi$ is greedy.
With the above, we can now turn to demonstrate $(b,c)$-selectability. 

\item[] \textbf{Selectability.} Upon arrival of element $e \in \cE$, let us consider $S$ and  $S^{\bd}$ to be the sets of elements already selected by $\pi$ and  $\hat \pi$,  respectively. By the way in which the constraint families are defined, and by construction of $\hat \pi$, we can observe that, given $\bx\in b\cP_{\cF}^{\bd}$ and $\by\in b\cP_{\cF}$, for all $S \subseteq R(\by)$ such that $S \cup \{e\} \in \cF_{\pi, \by}$, there always exists a set $S^{\bd} \subseteq R(\bx)$ such that $(S^{\bd} \cap \cE_e) \cup \{e\} \in \cF_{\pi, \bx}$. This establishes an injection between the selected set under standard constraints, and its counterpart under temporal constraints. We observe that, for all $e \in \cE$ and $\bx \in b\cP^{\bd}_{\cF}$,
\begin{multline*}
    \Pr\mleft[S^{\bd} \cup \{e\} \in \cF^{\bd}_{\pi, \bx} \quad \forall S^{\bd} \subseteq R(\bx), S^{\bd} \in \cF^{\bd}_{\pi, \bx}\mright] = \\ \Pr\mleft[(S^{\bd} \cap \cE_e) \cup \{e\} \in \cF_{\pi, \bx} \,\, \forall S^{\bd} \subseteq R(\bx), S^{\bd} \cap \cE_e \in \cF_{\pi, \bx}^{\bd}\mright].
\end{multline*}
Hence, since for greedy OCRS $\pi$ and $\by \in b\cP_{\cF}$, we have that
$\Pr\left[S \cup \{e\} \in \cF_{\pi, \by} \,\, \forall S \subseteq R(\by), S \in \cF_{\pi, \by}\right] \geq c$,
we can conclude by the injection above that
\begin{multline*}
    \Pr\left[(S^{\bd} \cap \cE_e) \cup \{e\} \in \cF_{\pi, \bx}\right.\\\left. \forall S^{\bd} \subseteq R(\bx), S^{\bd} \cap \cE_e \in \cF_{\pi, \bx}\right] \geq c.
\end{multline*}
\end{itemize}
\noindent The theorem follows.
\end{proof}

We remark that the above reduction is agnostic to the weight scale, i.e., we need not assume that $w_e \in [0,1]$ for all $e \in E$. In order to further motivate the significance of Algorithm \ref{alg:spattemp} and  Theorem \ref{thm:red}, in the Appendix we explicitly reduce the standard setting to the fully dynamic one for single-choice, and provide a general recipe for all packing constraints.
\section{Fully Dynamic Online Selection under Partial Information}\label{sec:bandit}

In this section, we study the case in which the decision-maker has to act under partial information. In particular, we focus on the following online sequential extension of the full-information problem: at each stage $t\in [T]$, a decision maker faces a new instance of the fully dynamic online selection problem. 
An unknown vector of weights $\bw_t\in[0,1]^{|E|}$ is chosen by an adversary at each stage $t$, while feasibility set $\cF^{\bd}$ is known and fixed across all $T$ stages. 
This setting is analogous to the one recently studied by \citet{gergatsouli2022online} in the context of \emph{Pandora's box problems}. 
A crucial difference with the online selection problem with full-information studied in Section~\ref{sec:temporal greedy ocrs} is that, at each step $t$, the decision maker has to decide whether to select or discard an element before observing its weight.
In particular, at each $t$, the decision maker takes an action $\ba_t\defeq \bone_{S^{\bd}_t}$, where $S^{\bd}_t\in \cF^{\bd}$ is the feasible set selected at stage $t$. The choice of $\ba_t$ is made before observing $\bw_t$.
The objective of maximizing the cumulative sum of weights is encoded in the reward function $f:[0,1]^{2m}\ni(\ba,\bw)\mapsto \langle \ba,\bw\rangle\in[0,1]$, which is the reward obtained by playing $\ba$ with weights $\bw = (w_{e})_{e \in \cE}$.~\footnote{The analysis can be easily extended to arbitrary functions linear in both terms.}

In this setting, we can think of $\cF^{\bd}$ as the set of \emph{super-arms} in a combinatorial online optimization problem. Our goal is designing online algorithms which have a performance \emph{close} to that of the best fixed super-arm in hindsight.\footnote{As we argue in  Appendix \ref{app:C} it is not possible to be competitive with respect to more powerful benchmarks.} In the analysis, as it is customary when the online optimization problem has an NP-hard offline counterpart, we resort to the notion of $\alpha$-regret.
In particular, given a set of feasible actions $\cX$, we define an algorithm's $\alpha$-regret up to time $T$ as
$$
    \Regret\defeq \alpha\,\max_{\bx\in\cX}\mleft\{\sum_{t=1}^T f(\bx,\bw_t)\mright\}-\E\mleft[\sum_{t=1}^T f(\bx_t,\bw_t)\mright],
$$
where $\alpha\in(0,1]$ and $\bx_t$ is the strategy output by the online algorithm at time $t$. We say that an algorithm has the \emph{no-$\alpha$-regret} property if $\Regret/T\to0$ for $T\to\infty$.

The main result of the section is providing a black-box reduction that yields a no-$\alpha$-regret algorithm for any fully dynamic online selection problem admitting a temporal OCRS. We provide no-$\alpha$-regret frameworks for three scenarios:
\begin{itemize}[leftmargin=*]
    \item \emph{full-feedback model}: after selecting $\ba_t$ the decision-maker observes the exact reward function $f(\cdot,\bw_t)$.
    
    \item \emph{semi-bandit feedback with white-box OCRS}: after taking a decision at time $t$, the algorithm observes $w_{t,e}$ for each element $e\in S^{\bd}_t$ (i.e., each element selected at $t$). Moreover, the decision-maker has exact knowledge of the procedure employed by the OCRS, which can be easily simulated. 
    
    \item \emph{semi-bandit feedback with oracle access to the OCRS}: the decision maker has semi-bandit feedback and the OCRS is given as a black-box which can be queried once per step $t$.
\end{itemize}

\subsection{Full-feedback Setting}

In this setting, after selecting $\ba_t$, the decision-maker gets to observe the reward function $f(\cdot,\bw_t)$. In order to achieve performance close to that of the best fixed super-harm in hindsight the idea is to employ the $\alpha$-competitive OCRS designed in Section~\ref{sec:temporal greedy ocrs} by feeding it with a fractional solution $\bx_t$ computed by considering the weights selected by the adversary up to time $t-1$.\footnote{We remark that a $(b,c)$-selectable OCRS yields a $bc$ competitive ratio. In the following, we let $\alpha\defeq bc$.}

Let us assume to have at our disposal a no-$\alpha$-regret algorithm for decision space $\cP^{\bd}_{\cF}$. We denote such regret minimizer as $\regmin$, and we assume it offers two basic operations: i) $\regmin.\rec()$ returns a vector in $\cP^{\bd}_{\cF}$; ii) $\regmin.\update(f(\cdot,\bw))$ updates the internal state of the regret minimizer using feedback received by the environment in the form of a reward function $f(\cdot,\bw)$.
Notice that the availability of such component is not enough to solve our problem since at each $t$ we can only play a super-arm $\ba_t\in\{0,1\}^m$ feasible for $\cF^{\bd}$, and not the strategy $\bx_t\in\cP^{\bd}_{\cF}\subseteq [0,1]^m$ returned by $\regmin$.
The decision-maker can exploit the subroutine $\regmin$ together with a temporal greedy OCRS $\hat\pi$ by following Algorithm~\ref{alg:alg full info}. We can show that, if the algorithm employs a regret minimizer for $\cP^{\bd}_{\cF}$ with a sublinear cumulative regret upper bound of $\cumreg$, the following result holds.
 \begin{algorithm}[!tp]
 	\caption{\textsc{Full-Feedback Algorithm}}
 	\label{alg:alg full info}
 	    \textbf{Input:} $T$, $\cF^{\bd}$, temporal OCRS $\hat\pi$, subroutine $\regmin$\\
		Initialize $\regmin$ for strategy space $\cP^{\bd}_{\cF}$\label{line:compute strategy}\\
		\For{$t\in[T]$}{
		    $\bx_t\gets \regmin.\rec()$\\
		    $\ba_t\gets \textnormal{execute OCRS } \hat\pi \textnormal{ with input } \bx_t$\\ 
		    Play $\ba_t$, and subsequently observe $f(\cdot,\bw_t)$\\
		    $\regmin.\update(f(\cdot,\bw_t))$
		}
\end{algorithm}

\begin{restatable}{theorem}{regFullInfo}\label{thm:reg full info}
Given a regret minimizer $\regmin$ for decision space $\cP^{\bd}_{\cF}$ with cumulative regret upper bound $\cumreg$, and an $\alpha$-competitive temporal greedy OCRS, Algorithm~\ref{alg:alg full info} provides 
\[
\alpha \max_{S\in\cI^{\bd}}\sum_{t=1}^T f(\bone_S,\bw_t) - \E\mleft[\sum_{t=1}^T f(\ba_t,\bw_t)\mright]\le \cumreg.
\]
\end{restatable}

\noindent Since we are assuming the existence of a polynomial-time separation oracle for the set $\cP_{\cF}^{\bd}$, then the LP $\argmax_{\bx\in\cP^{\bd}_{\cF}} f( \bx, \bw)$ can be solved in polynomial time for any $\bw$. Therefore, we can instantiate a regret minimizer for $\cP_{\cF}^{\bd}$ by using, for example, \emph{follow-the-regularised-leader} which yields $\cumreg\le\tilde O(m\sqrt{T})$~\cite{orabona2019modern}.

\subsection{Semi-Bandit Feedback with White-Box OCRS}

In this setting, given a temporal OCRS $\hat\pi$, it is enough to show that we can compute the probability that a certain super-arm $\ba$ is selected by $\hat\pi$ given a certain order of arrivals at stage $t$ and a vector of weights $\bw$. 
If that is the case, we can build a no-$\alpha$-regret algorithm with regret upper bound of $\tilde O(m\sqrt{T})$ by employing Algorithm~\ref{alg:alg full info} and by instantiating the regret minimizer $\regmin$ as the \emph{online stochastic mirror descent} (OSMD) framework by \citet{audibert2014regret}. We observe that the regret bound obtained is this way is tight in the semi-bandit setting \cite{audibert2014regret}.
Let $q_t(e)$ be the probability with which our algorithm selects element $e$ at time $t$. Then, we can equip OSMD with the following unbiased estimator of the vector of weights: $\hat w_{t,e} \defeq w_{t,e} a_{t,e}/q_t(e)$.~\footnote{We observe that $\hat w_{t,e}$ is equal to 0 when $e$ has not been selected at stage $t$ because, in that case, $a_{t,e}=0$.} 
In order to compute $q_t(\cdot)$ we need to have observed the order of arrival at stage $t$, the weights corresponding to super-arm $\ba_t$, and we need to be able to compute the probability with which the OCRS selected $e$ at $t$. This the reason for which we talk about ``white-box'' OCRS, as we need to simulate ex post the procedure followed by the OCRS in order to compute $q_t(\cdot)$.
When we know the procedure followed by the OCRS, we can always compute $q_t(e)$ for any element $e$ selected at stage $t$, since at the end of stage $t$ we know the order of arrival, weights for selected elements, and the initial fractional solution $\bx_t$. 
We provide further intuition as for how to compute such probabilities through the running example of matching constraints. 

\begin{example}\label{ex:running3}
Consider Algorithm~\ref{alg:alg full info} initialized with the OCRS of \Cref{ex:running}. Given stage $t$, we can safely limit our attention to selected edges (i.e., elements $e$ such that $a_{t,e}=1$). Indeed, all other edges will either be unfeasible (which implies that the probability of selecting them is $0$), or they were not selected despite being feasible. Consider an arbitrary element $e$ among those selected. Conditioned on the past choices up to element $e$, we know that $e \in \ba_t$ will be feasible with certainty, and thus the (unconditional) probability it is selected is simply $q_t(e) = 1 - e^{-y_{t, e}}$. 
\end{example}

\subsection{Semi-Bandit Feedback and Oracle Access to OCRS}

\begin{algorithm}[tb]
 	\caption{\textsc{Semi-Bandit-Feedback Algorithm with Oracle Access to OCRS}}
 	\label{alg:alg semi bandit}
 	{\bf Input:} $T$, $\cF^{\bd}$, temporal OCRS $\hat \pi$, full-feedback algorithm $\regmin$ for decision space $\cP_{\cF}^{\bd}$\\
 	Let $Z$ be initialized as in Theorem~\ref{thm:semi bandit}, and initialize \textsc{RM} appropriately\\
 	\For{$\tau=1,\ldots,Z$}{
 	    $I_\tau\gets\mleft\{(\tau-1)\frac{T}{Z}+1,\ldots,\tau\frac{T}{Z}\mright\}$\\
 	    Choose a random permutation $p:[m]\to \cE$, and $t_1,\ldots,t_{m}$ stages at random from $I_\tau$\\
 	    $\bx_\tau\gets\textsc{RM}.\rec()$\\
 	    \For{$t=(\tau-1)\frac{T}{Z}+1,\ldots,\tau\frac{T}{Z}$}{
 	        \If{$t=t_j$ \textnormal{for some} $j\in[m]$}{
 	            $\bx_t\gets \bone_{S^{\bd}}$ for a feasible set $S^{\bd}$ containing $p(j)$
 	        }\Else{
 	            $\bx_t\gets\bx_\tau$
 	        }
 	    Play $\ba_t$ obtained from the OCRS $\hat \pi$ executed with fractional solution $\bx_t$\\
 	    }
 	    Compute estimators $\tilde f_\tau(e)$ of $ f_\tau(e)\defeq \frac{1}{|I_\tau|}\sum_{t\in I_\tau} f(\bone_e,\bw_t)$ for each $e\in \cE$\\
 	    $\regmin.\update\mleft(\tilde f_\tau(\cdot)\mright)$
 	}
\end{algorithm}

As in the previous case, at each stage $t$ the decision maker can only observe the weights associated to each edge selected by $\ba_t$. Therefore, they have no counterfactual information on their reward had they selected a different feasible set. On top of that, we assume that the OCRS is given as a black-box, and therefore we cannot compute ex post the probabilities $q_t(e)$ for selected elements. 
However, we show that it is possible to tackle this setting by exploiting a reduction from the semi-bandit feedback setting to the full-information feedback one. In doing so, we follow the approach first proposed by  \citet{awerbuch2008online}.
The idea is to split the time horizon $T$ into a given number of equally-sized blocks. Each block allows the decision maker to simulate a single stage of the full information setting. 
We denote the number of blocks by $Z$, and each block $\tau\in [Z]$ is composed by a sequence of consecutive stages $I_\tau$.
Algorithm~\ref{alg:alg semi bandit} describes the main steps of our procedure. In particular, the algorithm employs a procedure \textsc{RM}, an algorithm for the full feedback setting as the one described in the previous section, that exposes an interface with the two operation of a traditional regret minimizer. 
During each block $\tau$, the full-information subroutine is used to compute a vector $\bx_\tau$. Then, in most stages of the window $I_\tau$, the decision $\ba_t$ is computed by feeding $\bx_\tau$ to the OCRS. A few stages are chosen uniformly at random to estimate utilities provided by other feasible sets (\emph{i.e.}, exploration phase). 
After the execution of all the stages in the window $I_\tau$, the algorithm computes estimated reward functions and uses them to update the full-information regret minimizer. 

Let $p:[m]\to\cE$ be a random permutation of elements in $\cE$. Then, for each $e\in \cE$, by letting $j$ be the index such that $p(j)=e$ in the current block $\tau$, an unbiased estimator $\tilde f_\tau(e)$ of $ f_\tau(e)\defeq \frac{1}{|I_\tau|}\sum_{t\in I_\tau} f(\bone_e,\bw_t)$ can be easily obtained by setting $\tilde f_\tau(e)\defeq f(\bone_e,\bw_{t_j})$. Then, it is possible to show that our algorithm provides the following guarantees.

\begin{restatable}{theorem}{semiBandit}\label{thm:semi bandit}
Given a temporal packing feasibility set $\cF^{\bd}$, and an $\alpha$-competitive OCRS $\hat \pi$, let $Z=T^{2/3}$, and the full feedback subroutine \textsc{RM} be defined as per Theorem \ref{thm:reg full info}. Then Algorithm~\ref{alg:alg semi bandit} guarantees that  
\[
\alpha \max_{S\in\cI^{\bd}}\sum_{t=1}^T f(\bone_S,\bw_t) - \E\mleft[\sum_{t=1}^T f(\ba_t,\bw_t) \mright] \le \tilde O(T^{2/3}).
\]
\end{restatable}
\section{Conclusion and Future Work}\label{sec:concl}
In this paper we introduce fully dynamic online selection problems in which selected items affect the combinatorial constraints during their activity times. We presented a generalization of the OCRS approach that provides near optimal competitive ratios in the full-information model, and no-$\alpha$-regret algorithms with polynomial per-iteration running time with both full- and semi-bandit feedback. 
Our framework opens various future research directions. For example, it would be particularly interesting to understand whether a variation of Algorithms \ref{alg:alg full info} and \ref{alg:alg semi bandit} can be extended to the case in which the adversary changes the constraint family at each stage. 
Moreover, the study of the bandit-feedback model remains open, and no regret bound is known for that setting.

\clearpage
\section*{Acknowledgements} The authors of Sapienza are supported by the Meta Research grant on ``Fairness and Mechanism Design'', the ERC Advanced Grant 788893 AMDROMA ``Algorithmic and Mechanism Design Research in  Online Markets'',   the MIUR PRIN project ALGADIMAR ``Algorithms, Games, and Digital Markets''.

\bibliography{sample}

\newpage
\appendix
\onecolumn

\clearpage
\appendix

\section{Contention Resolution Schemes and Online Contention Resolution Schemes}\label{app:ocrs}

As explained at length in \Cref{sec:prel}, our goal in general is that of finding the independent set of maximum weight for a given feasibility constraint family. However, doing this directly might be intractable in general and we need to aim for a good approximation of the optimum. 
In particular, given a non-negative submodular function $f: [0,1]^m \rightarrow \mathbb{R}_{\ge0}$, and a family of packing constraints $\cF$, we start from an \emph{ex ante} feasible solution to the linear program $\max_{\bx\in\cP_\cF} f(\bx)$, which upper bounds the optimal value achievable.
An \emph{ex ante} feasible solution is simply a distribution over the independent sets of $\cF$, given by a vector $\bx$ in the packing constraint polytope of $\cF$. 
A key observation is that we can interpret the \emph{ex ante} optimal solution to the above linear program as a vector $\bx^*$ of fractional values, which induces distribution over elements such that $x^*_e$ is the marginal probability that element $e \in \cE$ is included in the optimum. Then, we use this solution to obtain a feasible solution that suitably approximates the optimum. 
The random set $R(\bx^*)$ constructed by \emph{ex ante} selecting each element independently with probability $x^*_e$ can be infeasible. \emph{Contention Resolution Schemes} \cite{CVZ11} are procedures that, starting from the random set of \emph{sampled} elements $R(\bx^*)$, construct a \emph{feasible} solution with good approximation guarantees with respect to the optimal solution of the original integer linear program.

\begin{definition}[Contention Resolution Schemes (CRSs) \cite{CVZ11}]
For $b,c \in [0,1]$, a $(b,c)$-balanced Contention Resolution Scheme (CRS) $\pi$ for $\cF=(\cE,\cI)$ is a procedure such that, for every ex-ante feasible solution $\bx \in b\cP_\cF$ (i.e., the down-scaled version of polytope $\cP_\cF$), and every subset $S \subseteq \cE$, returns a random set $\pi(\bx, S) \subseteq S$ satisfying the following properties:
\begin{enumerate}
    \item Feasibility: $\pi(\bx, S) \in \cI$.
    \item $c$-balancedness: $\Pr_{\pi, R(\bx)}\left[e \in \pi(\bx, R(\bx)) \mid e \in R(\bx)\right] \geq c, \forall e \in \cE$.
\end{enumerate}
\end{definition}

When elements arrive in an online fashion, \citet{FSZ16} extend  CRS to the notion of OCRS, where $R(\bx)$ is obtained in the same manner, but elements are revealed one by one in adversarial order. The procedure has to decide irrevocably whether or not to add the current element to the final solution set, which needs to be feasible and a competitive against the offline optimum. The idea is that adding a \emph{sampled} element $e \in \cE$ to the set of already selected elements $S \subseteq R(\bx)$ maintains feasibility with at least constant probability, regardless of the element and the set. This originates \Cref{def:OCRS} and the subsequent discussion.

\section{Examples}\label{app:A}
In this section, we provide some clarifying examples for the concepts introduced in Section~\ref{sec:prel} and~\ref{sec:temporal}.

\subsection{Polytopes}\label{app:A1}

Example~\ref{ex: polytopes} provides the definition of the constraint polytopes of some standard problems, while Example~\ref{ex: polytopes-off} describes their \emph{temporal} version. For a set $S\subseteq \cE$ and $\bx\in\mathbb{R}^m$, we define, with a slight abuse of notation, $\bx(S) := \sum_{e \in S}x_e$.

\begin{example}[Standard Polytopes]\label{ex: polytopes} Given a ground set $\cE$,
\begin{itemize}
    \item Let $\cK = (\cE, \cI)$ be a knapsack constraint. Then, given budget $B > 0$ and a vector of elements' sizes $\bc\in\mathbb{R}^m_{\ge 0}$, its feasibility polytope is defined as
    \begin{align*}
        \cP_\cK = \left\{\bx \in [0, 1]^{m} : \langle \bc, \bx\rangle \leq B\right\}.
    \end{align*}
    \item Let $\cG = (\cE, \cI)$ be a matching constraint. Then, its feasibility polytope is defined as
    \begin{align*}
        \cP_\cG = \left\{\bx \in [0, 1]^{m} :\bx(\delta(u)) \leq 1 , \forall u \in V\right\},
    \end{align*}
    where $\delta(u)$ denotes the set of all adjacent edges to $u \in \cV$. Note that the ground set in this case is the set of all edges of graph $G = (\cV, \cE)$. 
    \item Let $\cM = (\cE, \cI)$ be a matroid constraint. Then, its feasibility polytope is defined as
    \begin{align*}
        \cP_\cM = \left\{\bx \in [0, 1]^{m} :\bx(S) \leq \emph{\text{rank}}(S) , \forall S \subseteq \cE\right\}.
    \end{align*}
    Here, $\emph{\text{rank}}(S) \defeq \max\left\{|I| : I \subseteq S, I \in \cI \right\}$, \emph{i.e.}, the cardinality of the maximum independent set contained in $S$.
\end{itemize}
\end{example}

We can now rewrite the above polytopes under \emph{temporal} packing constraints.
\begin{example}[\emph{Temporal} Polytopes]\label{ex: polytopes-off}
For ground set $\cE$,
\begin{itemize}
    \item Let $\cK = (\cE, \cI)$ be a knapsack constraint. Then, for $B > 0$ and cost vector $\bc\in\mathbb{R}^m_{\ge 0}$, its feasibility polytope is defined as
    \begin{align*}
        \cP^{\bd}_\cK = \left\{\bx \in [0, 1]^{m} :\langle \bc,\bx\rangle \leq B, \forall e \in \cE\right\}.
    \end{align*}
    \item Let $\cG = (\cE, \cI)$ be a matching constraint. Then, its feasibility polytope is defined as
    \begin{align*}
        \cP^{\bd}_\cG = \left\{\bx \in [0, 1]^{m} :\bx(\delta(u) \cap \cE_e) \leq 1 , \forall u \in \cV, \forall e \in \cE\right\}.
    \end{align*}
    \item Let $\cM = (\cE, \cI)$ be a matroid constraint. Then, its feasibility polytope is defined as
    \begin{align*}
        \cP^{\bd}_\cM = \left\{\bx \in [0, 1]^{m} :\bx(S \cap \cE_e) \leq \emph{\text{rank}}(S) , \forall S \subseteq \cE, \forall e \in \cE\right\}.
    \end{align*}
\end{itemize}
\end{example}

We also note that, for general packing constraints, if $d_e = \infty$ for all $e \in \cE$, then $\cE_e = \cE$, $\cP^{\infty}_\cF = \cP_\cF$, and similarly for the constraint family $\cF^\infty = \cF$.

\subsection{From Standard OCRS to Temporal OCRS for Rank-1 Matroids, Matchings, Knapsacks, and General Matroids}\label{app:A2}
In this section, we explicitly derive a $(1, 1/e)$-selectable (randomized) temporal greedy OCRS for the rank-1 matroid feasibility constraint, from a $(1, 1/e)$-selectable (randomized) greedy OCRS in the standard setting \cite{Liv21}, which is also tight. Let us denote this standard OCRS as $\pi_\cM$, where $\cM$ is a rank-1 matroid.
    
\begin{corollary}\label{cor:rank1}
For the rank-1 matroid feasibility constraint family under temporal constraints, Algorithm \ref{alg:spattemp} produces a $(1, 1/e)$-selectable (randomized) temporal greedy OCRS $\hat \pi_\cM$ from $\pi_\cM$.
\end{corollary}
\begin{proof}
Since it is clear from context, we drop the dependence on $\cM$ and write $\pi, \hat \pi$. We will proceed by comparing side-by-side what happens in $\pi$ and in $\hat \pi$. Let us recall from Examples \ref{ex: polytopes}, \ref{ex: polytopes-off} that the polytopes can respectively be written as
\begin{align*}
    \cP_{\cM} &= \left\{\bx \in [0, 1]^{m} :\bx(S) \leq 1 , \forall S \subseteq \cE\right\},\\
    \cP^{\bd}_{\cM} &= \left\{\by \in [0, 1]^{m} :\by(S \cap \cE_e) \leq 1 , \forall S \subseteq \cE, \forall e \in \cE\right\}.
\end{align*}

The two OCRSs perform the following steps, on the basis of Algorithm \ref{alg:spattemp}. On one hand, $\pi$ defines a subfamily of constraints $\cF_{\pi, \bx} \defeq \left\{\{e\} : e \in H(\bx)\right\}$, where $e \in \cE$ is included in random subset $H(\bx) \subseteq \cE$ with probability $\frac{1 - e^{-x_e}}{x_e}$. Then, it selects the first sampled element $e \in R(\bx)$ such that $\{e\} \in \cF_{\pi, \bx}$. On the other hand, $\pi_{\by}$ defines a subfamily of constraints $\cF^{\bd}_{\pi, \by} \defeq \left\{\{e\} : e \in H(\by)\right\}$, where $e \in \cE$ is included in random subset $H(\by) \subseteq \cE$ with probability $q_e(\by) = \frac{1 - e^{-y_e}}{y_e}$. The feasibility family $\cF^{\bd}_{\pi, \by}$ induces, as per Observation \ref{obs:inf}, a sequence of feasibility families $\cF_{\pi, \by}(e) \defeq \left\{\{e\} : e \in H(\by) \cap \cE_e\right\}$, for each $e \in \cE$. For all $e^\prime \in \cE$, the OCRS selects the first sampled element $e \in R(\by)$ such that $\{e\} \in \cF_{\pi, \by}(e)$. In other words, the temporal OCRS selects a sampled element that is active only if no other element in its active elements set has been selected earlier. It is clear that both are randomized greedy OCRSs.

We will now proceed by showing that each element $e$ is selected with probability at least $1/e$ in both $\pi, \hat\pi$. In $\pi$ element $e$ is selected if sampled, and no earlier element has been selected before (i.e. its singleton set belongs to the subfamily $\cF_{\pi, \bx}$). An element $e^\prime$ is not selected with probability $1 - x_{e^\prime} \cdot \frac{1 - e^{-x_{e^\prime}}}{x_{e^\prime}} = e^{-x_{e^\prime}}$. This means that the probability of $e$ being selected is
\begin{align*}
    \frac{1 - e^{-x_e}}{x_e} \cdot \prod_{s_{e^\prime} < s_e}{e^{-x_{e^\prime}}} = \frac{1 - e^{-x_e}}{x_e} \cdot e^{-\sum_{s_{e^\prime} < s_e}x_{e^\prime}} \geq \frac{\left(1 - e^{-x_e}\right)e^{x_e-1}}{x_e} \geq \frac{1}{e},
\end{align*}
where the first inequality is justified by $\sum_{s_{e^\prime} < s_e}x_{e^\prime} \leq 1$, and the second follows because the expression is minimized for $x_e = 0$. Similarly, in $\hat\pi$ element $e$ is selected if sampled, and no earlier element \emph{that is still active} has been selected before (i.e. its singleton set belongs to the subfamily $\cF_{\pi, \by}(e)$). We have that the probability of $e$ being selected is
\begin{align*}
    \frac{1 - e^{-y_e}}{y_e} \cdot \prod_{s_{e^\prime} < s_e: e^\prime \in \cE_e}{e^{-y_{e^\prime}}} = \frac{1 - e^{-y_e}}{y_e} \cdot e^{-\sum_{s_{e^\prime} < s_e: e^\prime \in \cE_e}y_{e^\prime}} \geq \frac{\left(1 - e^{-y_e}\right)e^{y_e-1}}{y_e} \geq \frac{1}{e}.
\end{align*}
Again, the first inequality is justified by $\sum_{s_{e^\prime} < s_e: e^\prime \in \cE_e}y_{e^\prime} \leq 1$ by the temporal feasibility constraints, and the second follows because the expression is minimized for $y_e = 0$. Selectability is thus shown.
\end{proof}

\begin{remark}\label{rem:recipe}
Adapting the OCRSs in Theorem 1.8 of \cite{FSZ16} for general matroids, matchings and knapsacks, by following Algorithm \ref{alg:spattemp} step-by-step, we get the same selectability guarantees in the temporal settings as in the standard ones: respectively, $(b, 1-b), (b, e^{-2b}), (b, (1-2b)/(2-2b))$. There are two crucial steps to map a standard OCRS into a temporal one, as exemplified by Corollary \ref{cor:rank1}:
\begin{enumerate}
    \item We first need to define the temporal constraints based on the standard ones. This is done simply by enforcing the constraint in standard setting only for the current set of active elements, i.e. transforming $\cF_{\pi, \bx}$ into $\cF_{\pi, \by}(e)$ for all elements $e \in \cE$. Such a transformation is analogous to the one used to go from Example \ref{ex: polytopes} to Example \ref{ex: polytopes-off}.
    \item When proving selectability, the probability of feasibility is only calculated on elements $e^\prime$ belonging the same independent set as $e$ (which arrives later), that are \emph{still active}. This means that the probability computation is confined to only $e^\prime \in \cE_e$ such that $s_{e^\prime} < s_e$, rather than all $e^\prime \in \cE$ such that $s_{e^\prime} < s_e$.
\end{enumerate}
\end{remark}

\section{Batched Arrival: Matching Constraints}\label{app:B}
As mentioned in \Cref{sec:intro}, \citet{EFGT20} generalize the one-by-one online selection problem to a setting where elements arrive in batches. The existence of batched greedy OCRSs implies a number of results, as for instance Prophet Inequalities under matching constraints where, rather than edges, vertices with all the edges adjacent to them arrive one at a time. This can be viewed as an incoming batch of edges, for which \citet{EFGT20} explicitly construct a $(1, 1/2)$-selectable batched greedy OCRS.

Indeed, we let the ground set $\cE$ be partitioned in $k$ disjoint subsets (batches) arriving in the order $B_1, \dots, B_k$, and where elements in each batch appear at the same time. Such batches need to belong to a feasible family of batches $\cB$: for example, all batches could be required to be singletons, or they could be required to be all edges incident to a given vertex in a graph, and so on. Similarly to the traditional OCRS, we sample a random subset $R_j(\bx) \subseteq B_j$, for all $j \in [k]$, so as to form $R(\bx) \defeq \bigcup_{j \in [k]}R_j(\bx) \subseteq \cE$, where $R_j$'s are mutually independent. The fundamental difference with greedy OCRSs is that, within a given batch, weights are allowed to be correlated.

\begin{definition}[Batched Greedy OCRSs \cite{EFGT20}]\label{def:batchedOCRS}
For $b,c \in [0,1]$, let $\cP_\cF \subseteq [0,1]^{m}$ be $\cF$'s feasibility polytope. An OCRS $\pi$ for $b\cP_\cF$ is called a batched greedy OCRS with respect to $R$ if, for every ex-ante feasible solution $\bx \in b\cP_\cF$, $\pi$ defines a packing subfamily of feasible sets $\cF_{\pi, \bx} \subseteq \cF$, and it selects a sampled element $e \in B_j$ when, together with the set of already selected elements, the resulting set is in $\cF_{\pi, \bx}$. We say that a batched greedy OCRS $\pi$ is $(b,c)$-selectable if $\Pr_{\pi, R(\bx)}\left[S_j \cup \{e\} \in \cF_{\pi, \bx} \quad \forall S_j \subseteq R_j(\bx), S_j \in \cF_{\pi, \bx}\right] \geq c$, for each $j \in [k], e \in S_j$. The output feasible set will be $S \defeq \bigcup_{j \in [m]}S_j \in \cF_{\pi, \bx}$.
\end{definition}
Naturally, \Cref{thm:red} extends to batched OCRSs.

\begin{corollary}\label{cor:batchred}
Let $\cF, \cF^{\bd}$ be respectively the standard and \emph{temporal} packing constraint families, with their corresponding polytopes $\cP_{\cF}, \cP^{\bd}_{\cF}$. Let $\bx \in b\cP_{\cF}$ and $\by \in b\cP^{\bd}_{\cF}$, and consider a $(b,c)$-selectable batched greedy OCRS $\pi$ for $\cF_{\pi, \bx}$, with batches $B_1, \dots, B_k \in \cB$. We can construct a batched greedy OCRS $\hat\pi$ that is also $(b,c)$-selectable for $\cF^{\bd}_{\pi, \by}$, with batches $B_1, \dots, B_k \in \cB$.
\end{corollary}

The proof of this corollary is identical to that of \Cref{thm:red}: we can indeed define a set of active elements $\cE_j$ for each batch $B_j$, and $\hat\pi$ is essentially in Algorithm \ref{alg:spattemp} but with incoming batches rather than elements, and the necessary modifications in the sets. We will demonstrate the use of batched greedy OCRSs in the graph matching setting, where vertices come one at a time together with their contiguous edges. This allows us to solve the problem of dynamically assigning tasks to reviewers for the reviewing time, and to eventually match new tasks to the same reviewers, so as to maximize the throughput of this procedure. Details are presented in Appendix \ref{app:B}.

By Corollary \ref{cor:batchred} together with Theorem 4.1 in \citet{EFGT20}, which gives an explicit construction of a $(1,1/2)$-selectable batched greedy OCRS under matching constraints, we immediately have that $(1,1/2)$-selectable batched greedy OCRS exists even under temporal constraints. For clarity, and in the spirit of Appendix \ref{app:A2}, we work out how to derive from scratch an online algorithm that is $1/2$-competitive with respect to the offline optimal matching when the graph is bipartite and temporal constraints are imposed. We do not use of Corollary \ref{cor:batchred}, but we follow the proof of this general statement for the specific setting of bipartite graph matching. Batched OCRSs in the non-temporal case are not specific to the bipartite matching case but extend in principle to arbitrary packing constraints. Nevertheless, the only known constant competitive batched OCRS is the one for \emph{general} graph matching by \citet{EFGT20}. Finally, we note that our results closely resemble the ones of \citet{dickerson2018allocation}, with the difference that their arrival order is assumed to be stochastic, whereas ours is adversarial.

This is motivated for instance by the following real-world scenario: there are $|U| = m$ ``offline'' beds (machines) in an hospital, and $|V| = n$ ``online'' patients (jobs) that arrive. Once a patient $v \in V$ comes, the hospital has to irrevocably assign it to one of the beds, say $u \in U$, and occupy it for a stochastic time equal to $d_{uv} := d_v[u]$, for $d_v \sim \cD_{v}$, i.e., the $u$-th component of random vector $d_v$. The sequence of arrivals is adversarial, but with known ex-ante distributions $\left(\cW_{v}, \cD_{v}\right)$. Moreover, the patient's healing can be thought of as a positive reward/weight equal to $w_{uv} := w_v[u]$, for $w_v \sim \cW_{v}$, i.e., the $u$-th component of random vector $w_v$, whose distributions are known to the algorithm. The hospital's goal is that of maximizing the sum of the healing weights over time, i.e., over a discrete time period of length $|V| = n$. Across $v$'s, both $w_v$'s and $d_v$'s are independent. However, within the vector itself, components $d_{uv}$ and $d_{u^\prime v}$ could be correlated, and the same holds for $w_v$'s.

\subsection{Linear Programming Formulation}
First, we construct a suitable linear-programming formulation whose fractional solution yields an upper bound on the expected optimum offline algorithm. Then, we devise an online algorithm that achieves an $\alpha$-competitive ratio with respect to the linear programming fractional solution. We follow the \emph{temporal} LP Definition \label{def:tlp}, and let $f(\bx) \defeq \langle \bw, \bx \rangle$, for $\bx \in \cP^{\bd}_{\cG}$ being a feasible fractional solution in the matching polytope. Since the matching polytope is $\cP^{\bd}_\cG = \left\{\bx \in [0, 1]^{m} :\bx(\delta(u) \cap \cE_e) \leq 1 , \forall u \in V, \forall e \in \cE\right\}$, we can equivalently write the \emph{temporal} linear program as

\begin{equation}\label{eq:lp}\tag{1}
\mleft\{\begin{array}{lll}
		\displaystyle
		\max_{\bx \in [0,1]^{m}} &\displaystyle \sum_{u \in U}\sum_{v \in V}\overline{w}_{uv} \cdot x_{uv} &  \triangleleft \text{ Objective}\\
		\text{s.t.} &\displaystyle \sum_{u \in U} x_{uv} \leq 1, \ \forall v \in V &\triangleleft \text{ Constr. 1}\nonumber\\
		&\displaystyle\sum_{v^\prime: s_{v^\prime} < s_v}\hspace{-.3cm}
		x_{uv^\prime}\cdot\emph{\Pr}\left[d_{uv^\prime} \geq s_v - s_{v^\prime}\right] + x_{uv} \leq 1, \ \forall u \in U, v\in V &\triangleleft \text{ Constr. 2}\nonumber\\
		&x_{uv} \geq 0, \ \forall u \in U, v\in V &\triangleleft \text{ Constr. 3}\nonumber\\
	\end{array}\mright.
\end{equation}

\noindent where $\overline{w}_{uv} \defeq \E_{w_v \sim \cW_v}\left[w_{uv}\right]$, when $w_{uv}$ is a random variable; when instead, it is deterministic, we simply have $\overline{w}_{uv} = w_{uv}$. Furthermore, as we argued in \Cref{sec:prel}, we can think of $\bx_{uv}$ to be the probability that edge $uv$ is inserted in the offline (fractional) optimal matching. We now show why the above linear program yields an upper bound to the offline optimal matching.

\begin{lemma}\label{lp-opt}
Cosider solution $\bx^*$ to linear program (\ref{eq:lp}). Then, $\bx^*$ is such that $\langle \overline{\bw}, \bx^*\rangle \geq \E_{\bw, \bd}\left[\langle \bw, \bone_{\OPT}\rangle\right]$, where $\bone_{\OPT} \in \{0,1\}^{m}$ is the vector denoting which of the elements have been selected by the integral offline optimum.
\end{lemma}
\begin{proof}
The proof follows from analyzing the constraints. The meaning of Constraint 1 is that upon the arrival of vertex $v$, $v$ must be matched at most once in expectation. In fact, for each job $v \in V$, at most one machine $u \in U$ can be selected by the optimum, which yields
\begin{align*}
    \sum_{u \in U}{x_{uv}} \leq 1.
\end{align*}
This justifies Constraint 1. Constraint 2, on the other hand, has the following simple interpretation: machine $u$ is unavailable when job $v$ arrives if it has been matched earlier to a job $v^\prime$ such that the activity time is longer than the difference of $v, v^\prime$ arrival times. Otherwise, $u$ can in fact be matched to $v$, and this probability is of course lower than the probability of being available. This implies that for each machine $u \in U$ and each job $v \in V$,
\begin{align*}
    \sum_{v^\prime: s_{v^\prime} < s_v}&x_{uv^\prime}\cdot\emph{\Pr}\left[d_{uv^\prime} \geq s_v - s_{v^\prime}\right] + x_{uv} \leq 1.
\end{align*}
We have shown that all constraints are less restrictive for the linear program as they would be for the offline optimum. Since the objective function is the same for both, a solution for the integral optimum is also a solution for the linear program, while the converse does not necessarily hold. The statement follows.
\end{proof}

\subsection{A simple algorithm}
Inspired by the algorithm by \citet{dickerson2018allocation} (which deals with stochastic rather than adversarial arrivals), we propose Algorithm \ref{alg:matching-tocrs}. In the remainder, let $N_{\avail}(v)$ denote the set of available vertices $u \in U$ when $v \in V$ arrives.

\RestyleAlgo{ruled}

\begin{algorithm}
\caption{Bipartite Matching Temporal OCRS}\label{alg:matching-tocrs}
\KwData{Machine set $U$, job set $V$, and distributions $\cW_v, \cD_v$}
\KwResult{Matching $M \subseteq U \times V$}
Solve LP (\ref{eq:lp}) and obtain fractional solution solution $\bx^*$\;
$M \leftarrow \emptyset$\;
\For{$v \in V$}
{
    \eIf{$N_{\emph{\avail}}(v) = \emptyset$}
    {
        Reject $v$\;
    }{
        Select $u \in N_{\avail}(v)$ with probability $\alpha \cdot \frac{x^*_{uv}}{\emph{\Pr}\left[u \in N_{\avail}(v)\right]}$\;
        $M \leftarrow M \cup \{uv\}$\;
    }
}
\end{algorithm}

\begin{lemma}\label{compratio-prob}
Algorithm \ref{alg:matching-tocrs} makes every vertex $u \in U$ available with probability at least $\alpha$. Moreover, such probability is maximized for $\alpha = 1/2$.
\end{lemma}
\begin{proof}
We will prove the claim by induction. For the first incoming job $v = 1$, $\emph{\Pr}\left[u \in N_{\avail}(v)\right] = 1 \geq \alpha$ for all machines $u \in U$, no matter what the values of $w_{uv}, d_{uv}$ are. To complete the base case, we only need to check that the probability of selecting one machine is in fact no larger than one: for this purpose, let us name the event $u$ is selected by Algorithm \ref{alg:matching-tocrs} when $v$ comes as $u \in \ALG(v)$.
\begin{align*}
    \emph{\Pr}\left[\exists u \in N_{\avail}(v): \ u \in \ALG(v)\right] = \sum_{u \in U}{\alpha \cdot \frac{x^*_{uv}}{\emph{\Pr}\left[u \in N_{\avail}(v)\right]}} \leq \alpha,
\end{align*}
where the first equality follows from the fact the events within the existence quantifier are disjoint, and recalling that $N_{\avail}(v) = U$ for the first job. Consider all vertices $v^\prime$ arriving before vertex $v$ ($s_{v^\prime} < s_v$), and assume that $\emph{\Pr}\left[u \in N_{\avail}(v^\prime)\right] \geq \alpha$ always. This means that the algorithm is makes each $u$ available with probability at least $\alpha$ for all vertex arrivals before $v$. This, in turn, implies that each $u$ is selected with probability $\alpha \cdot x^*_{uv^\prime}$. Let us observe that a machine $u \in U$ will not be available for the incoming job $v \in V$ only if the algorithm has matched it to an earlier job $v^\prime$ with activity time larger than $s_v - s_{v^\prime}$. Formally, the probability that $u$ is available for $v$ is
\begin{align*}
    \emph{\Pr}\left[u \in N_{\avail}(v)\right] &= 1 - \emph{\Pr}\left[u \notin N_{\avail}(v)\right] \\
    &= 1 - \emph{\Pr}\left[\exists v^\prime \in V: s_{v^\prime} < s_v, \ u \in \ALG(v^\prime), d_{uv^\prime} > s_v - s_{v^\prime}\right] \\
    &\geq 1 - \alpha \cdot \sum_{v^\prime: s_{v^\prime} < s_v}{x^*_{uv^\prime}\emph{\Pr}\left[d_{uv^\prime} \geq s_v - s_{v^\prime}\right]} \\
    &\geq \alpha + \alpha \cdot x^*_{uv}\\
    &\geq \alpha 
\end{align*}
The second to last inequality follows from Constraint 2, and by observing the following simple implication for all $r, z \in \mathbb{R}$: if $r + z \leq 1$, then $1 - \alpha r \geq \alpha + \alpha z$, so long as $\alpha \leq \frac{1}{2}$. Since we would like to choose $\alpha$ as large as possible, we choose $\alpha = \frac{1}{2}$. What is left to be shown is that the probability of selecting one machine is at most one:
\begin{align*}
    \emph{\Pr}\left[\exists u \in N_{\avail}(v): \ u \in \ALG(v)\right] = \sum_{u \in U}{\alpha \cdot \frac{x^*_{uv}}{\emph{\Pr}\left[u \in N_{\avail}(v)\right]}} \leq 1.
\end{align*}
The statement, thus, follows.
\end{proof}

A direct consequence of the above two lemmata is the following theorem. Indeed, if every $u$ is available with at least probability $1/2$, then the algorithm will select it, regardless of what the previous algorithm actions. In turn, the optimum will be approximated with the same factor.

\begin{theorem}
Algorithm \ref{alg:matching-tocrs} is $\frac{1}{2}$-competitive with respect to the expected optimum $\E_{\bw, \bd}\left[\langle \bw, \bone_{\OPT}\rangle\right]$.
\end{theorem}

Various applications such as prophet and probing inequalities for the batched temporal setting can be derived from the above theorem. Solving them with a constant competitive ratio yields a solution for the review problem illustrated in the introduction, where multiple financial transactions arriving over time could be assigned to one of many potential reviewers, and these reviewers can be ``reused'' once they have completed their review time.

\section{Benchmarks}\label{app:C}

\subsection{The need for stages}
We argue that, for the results in Section \ref{sec:bandit}, stages are necessary in order for us to be able to compare our algorithm against any meaningful benchmark. Suppose, in contrast, that we chose to compare against the optimum (or an approximation of it) within a single stage where $n$ jobs arrive to a single arm. 
A \emph{non-adaptive} adversary could simply run the following procedure, with each job having weight $1$: with probability $1/2$, jobs with odd arrival order have activity time $1$, and jobs with even arrival order have activity time $\infty$, with probability $1/2$ the opposite holds. To be precise, let us recall that $\infty$ is just a shorthand notation to mean that all future jobs would be blocked: indeed, the activity time of a job arriving at time $s_e$ is not unbounded but can be at most $n-s_e$. As activity times are revealed after the algorithm has made a decision for the current job, the algorithm does not know whether taking the current job will prevent it from being blocked for the entire future. The best thing the algorithm can do is to pick the first job with probability $1/2$. Indeed, if the algorithm is lucky and the activity time is $1$ then it knows to be in the first scenario and gets $n$. Otherwise, it only gets $1$. Hence, the regret would be $\cumreg[n] = n - \frac{n+1}{2} \in \Omega(n)$, which is linear. Note that $n$ and $T$ here represent two different concepts: the first is the number of elements sent \emph{within} a stage; the second is the number of stages. In the case outlined above, $T = 1$, since it is a single stage scenario. Thus, there is no hope that in a single stage we could do anything meaningful, and we turn to the framework where an entire instance of the problem is sent at each stage $t \in [T]$. 

\subsection{Choosing the right benchmark}
Now, we motivate why the Best-in-Hindsight policy introduced at the beginning of Section \ref{sec:bandit} is a strong and realistic benchmark, for an algorithm that knows the feasibility polytopes a priori. In fact, when we want to measure regret, we need to find a benchmark to compare against, which is neither too trivial nor unrealistically powerful compared to the information we have at hand. 
Below, we provide explicit lower bounds which show that the dynamic optimum is a too powerful benchmark even when the polytope is known. In particular, the next examples prove that it is impossible to achieve sublinear ($\alpha$-)Regret against the dynamic optimum. In the remainder, we always assume \emph{full feedback} and that the adversary is \emph{non-adaptive}.

In the remainder, we denote by $\ba_t^{\OPT}$ and $\ba_t^{\ALG}$ the action chosen at time $t$ by the optimum and the algorithm respectively.

\begin{lemma}\label{lem:regret}
Every algorithm has $\cumreg = \sum_{t \in [T]}\E[f_t(\ba_t^{\OPT})] - \sum_{t \in [T]}\E[f_t(\ba_t^{\ALG})] \in \Omega(T)$ against the dynamic optimum.
\end{lemma}
\begin{proof}
Consider the case of a single arm and the arrival of 3 jobs at each stage (on at a time within the stage, revealed from top to bottom), with the constraint that at most 1 active job can be selected. The (\emph{non-adaptive}) adversary simply tosses $T$ fair coins independently at each stage: if the $t^{\text{th}}$ coin lands heads, then all $3$ jobs at the $t^{\text{th}}$ stage have activity times $1$ and weights $1$, otherwise all jobs have activity time $\infty$, the first job has weight $\epsilon$ and the last two have weight $1$ (recall that $\infty$ is just a shorthand notation to mean that all future jobs would be blocked). Figure \ref{fig:im1} shows a possible realization of the $T$ stages: at each stage the expected reward of the optimal policy is $\frac{3}{2}$, since the optimal value is $1$ or $2$ with equal probability. By linearity of expectation, $\sum_{t \in [T]}\E[f_t(\ba_t^{\OPT})] = T \cdot \E[f(\ba^{\OPT})] \geq \frac{3}{2}T$. 
\begin{center}
    \begin{figure}[h]
    \hspace{2.5cm}
        \scalebox{.75}{\tikzset{every picture/.style={line width=0.75pt}} %
\begin{tikzpicture}[x=0.75pt,y=0.75pt,yscale=-1,xscale=1]
    \draw   (12,138) -- (91,138) -- (91,178) -- (12,178) -- cycle ;
    \draw    (91,158) -- (117,68) ;
    \draw    (91,158) -- (121,158) ;
    \draw    (91,158) -- (114,239) ;
    \draw   (117,68) .. controls (117,61.1) and (122.6,55.5) .. (129.5,55.5) .. controls (136.4,55.5) and (142,61.1) .. (142,68) .. controls (142,74.9) and (136.4,80.5) .. (129.5,80.5) .. controls (122.6,80.5) and (117,74.9) .. (117,68) -- cycle ;
    \draw   (121,158) .. controls (121,151.1) and (126.6,145.5) .. (133.5,145.5) .. controls (140.4,145.5) and (146,151.1) .. (146,158) .. controls (146,164.9) and (140.4,170.5) .. (133.5,170.5) .. controls (126.6,170.5) and (121,164.9) .. (121,158) -- cycle ;
    \draw   (114,239) .. controls (114,232.1) and (119.6,226.5) .. (126.5,226.5) .. controls (133.4,226.5) and (139,232.1) .. (139,239) .. controls (139,245.9) and (133.4,251.5) .. (126.5,251.5) .. controls (119.6,251.5) and (114,245.9) .. (114,239) -- cycle ;
    \draw   (177,138) -- (256,138) -- (256,178) -- (177,178) -- cycle ;
    \draw    (256,158) -- (282,68) ;
    \draw    (256,158) -- (286,158) ;
    \draw    (256,158) -- (279,239) ;
    \draw   (282,68) .. controls (282,61.1) and (287.6,55.5) .. (294.5,55.5) .. controls (301.4,55.5) and (307,61.1) .. (307,68) .. controls (307,74.9) and (301.4,80.5) .. (294.5,80.5) .. controls (287.6,80.5) and (282,74.9) .. (282,68) -- cycle ;
    \draw   (286,158) .. controls (286,151.1) and (291.6,145.5) .. (298.5,145.5) .. controls (305.4,145.5) and (311,151.1) .. (311,158) .. controls (311,164.9) and (305.4,170.5) .. (298.5,170.5) .. controls (291.6,170.5) and (286,164.9) .. (286,158) -- cycle ;
    \draw   (279,239) .. controls (279,232.1) and (284.6,226.5) .. (291.5,226.5) .. controls (298.4,226.5) and (304,232.1) .. (304,239) .. controls (304,245.9) and (298.4,251.5) .. (291.5,251.5) .. controls (284.6,251.5) and (279,245.9) .. (279,239) -- cycle ;
    \draw   (508,138) -- (587,138) -- (587,178) -- (508,178) -- cycle ;
    \draw    (587,158) -- (613,68) ;
    \draw    (587,158) -- (617,158) ;
    \draw    (587,158) -- (610,239) ;
    \draw   (613,68) .. controls (613,61.1) and (618.6,55.5) .. (625.5,55.5) .. controls (632.4,55.5) and (638,61.1) .. (638,68) .. controls (638,74.9) and (632.4,80.5) .. (625.5,80.5) .. controls (618.6,80.5) and (613,74.9) .. (613,68) -- cycle ;
    \draw   (617,158) .. controls (617,151.1) and (622.6,145.5) .. (629.5,145.5) .. controls (636.4,145.5) and (642,151.1) .. (642,158) .. controls (642,164.9) and (636.4,170.5) .. (629.5,170.5) .. controls (622.6,170.5) and (617,164.9) .. (617,158) -- cycle ;
    \draw   (610,239) .. controls (610,232.1) and (615.6,226.5) .. (622.5,226.5) .. controls (629.4,226.5) and (635,232.1) .. (635,239) .. controls (635,245.9) and (629.4,251.5) .. (622.5,251.5) .. controls (615.6,251.5) and (610,245.9) .. (610,239) -- cycle ;
    \draw (84,88.4) node [anchor=north west][inner sep=0.75pt]    {$1,1$};
    \draw (101,137.4) node [anchor=north west][inner sep=0.75pt]    {$1,1$};
    \draw (80,199.4) node [anchor=north west][inner sep=0.75pt]    {$1,1$};
    \draw (248,88.4) node [anchor=north west][inner sep=0.75pt]    {$\epsilon,\infty$};
    \draw (264,137.4) node [anchor=north west][inner sep=0.75pt]    {$1,\infty$};
    \draw (241,199.4) node [anchor=north west][inner sep=0.75pt]    {$1,\infty$};
    \draw (578,88.4) node [anchor=north west][inner sep=0.75pt]    {$\epsilon,\infty$};
    \draw (592,137.4) node [anchor=north west][inner sep=0.75pt]    {$1,\infty$};
    \draw (572,199.4) node [anchor=north west][inner sep=0.75pt]    {$1,\infty$};
    \draw (384,151.4) node [anchor=north west][inner sep=0.75pt]    {$\dots \dots \dots $};
\end{tikzpicture}}
        \caption{Three jobs per stage: w.p. $1/2$, either $\{(1,1), (1,1), (1,1)\}$ or $\{(\epsilon, \infty), (1, \infty), (1, \infty)\}$.}
        \label{fig:im1}
    \end{figure}
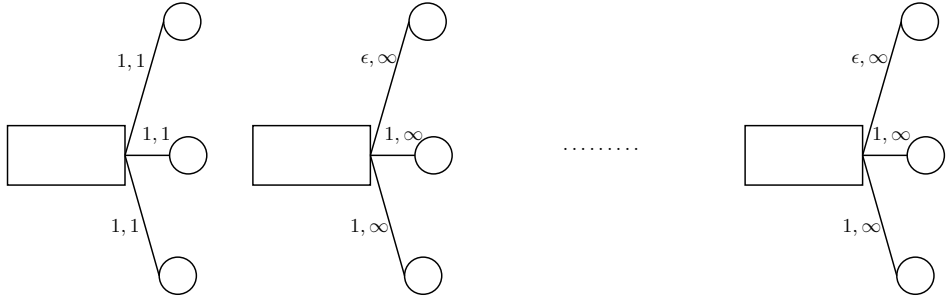
\end{center}
On the other hand, the algorithm will discover which scenario it has landed into only after the value of the first job has been revealed. If it does not pick it and it results in a weight of $1$, then the algorithm can get at most $1$ from the remaining jobs. If instead it decides to pick it but it realizes in an $\epsilon$ value, it will only get $\epsilon$. Even if the algorithm is aware of such a stochastic input beforehand, it knows that stages are independent and, hence, cannot be adaptive before a given stage begins. Then, it observes the first job weight without taking it, but it may already be too late. Any algorithm in this setting can be described by deciding to accept the first job with probability $p$ (and reject it with $1-p$), and then act adaptively. Then, again by linearity of expectation,
\begin{align*}
    \sum_{t \in [T]}\E[f_t(\ba_t^{\ALG})] &= T \cdot \E[f(\ba^{\ALG})] = T \cdot \left(\frac{1}{2}\left(2p + (1-p)\right) + \frac{1}{2}\left(\epsilon p + (1-p)\right)\right) \\
    &= \frac{2+\epsilon p}{2} \cdot T \leq (1+\epsilon) \cdot T.
\end{align*}
Thus, $\cumreg \geq (1-\epsilon) \cdot T \in \Omega(T)$.
\end{proof}

Now, we ask whether there exists a similar lower bound on approximate regret. Similarly to the previous lemma, we denote by $\ba_t^{\OCRS}$ the 
action chosen at time $t$ by the OCRS.
\begin{lemma}\label{lem:asumopt}
Every algorithm has $\cumreg = \alpha \cdot \sum_{t \in [T]}\E[f_t(\ba_t^{\OPT})] - \sum_{t \in [T]}\E[f_t(\ba_t^{\ALG})] \in \Omega(T)$ against an $\alpha$-approximation of the dynamic optimum, for $\alpha \in (0,1]$.
\end{lemma}
\begin{proof}
Let all the activity times be infinite, and define (for a given stage) the constraint to be picking a single job irrevocably. We know that, for the single-choice problem, a tight OCRS achieves $\alpha = 1/2$ competitive ratio. However, such OCRS  is not greedy. \citet{Liv21} constructs a tight greedy OCRS for single-choice, which is $\alpha = 1/e$ competitive. For our purposes, nonetheless, we only require the trivial inequality $\alpha \leq 1$. The \emph{non-adaptive} adversary could run the following a priori procedure, for each of the $T$ stages: let $\delta = \frac{\alpha - 1/n}{2}$ be a constant, sample $k \sim [n]$ uniformly at random, and send jobs in order of weights $\delta^k, \delta^{k-1}, \dots, \delta, 0, \dots, 0$ (ascending until $\delta$ and then all $0$s).\footnote{This construction is inspired by the notes of \citet{KNotes16}.} We know that, by Theorem \ref{thm:red},
\begin{align*}
    \sum_{t \in [T]}\E[f_t(\ba_t^{\OCRS})] \geq \alpha \cdot \sum_{t \in [T]}\E[f_t(\ba_t^{\OPT})].
\end{align*}
 This is possible because the greedy OCRS has access full-information about the current stage a priori (it knows $\delta$ and the sampled $k$ at each stage), unlike the algorithm, which is unaware of how the $T$ stages are going to be presented. It is easy to see that what the best the algorithm can do within a given stage is to randomly guess what the drawn $k$ has been, i.e., where $\delta$ will land. We now divide the time horizon in $T/n$ intervals, each composed of $n$ stages. In each interval, since no stage is predictive of the next, we know that the algorithm cannot be adaptive across stages, nor can it be within a stage, since  all possible sequences have the same prefix. By construction, we expect the algorithm to catch $\delta$ once per time interval, and otherwise get at most $\delta^2$, optimistically for all remaining $n-1$ stages. In other words, let us index each interval by $I \in [T/n]$ and rewrite the algorithm and the OCRS expected rewards as
\begin{align*}
    &\sum_{t \in [T]}\E[f_t(\ba_t^{\ALG})] \leq \sum_{I \in [T/n]}\left(\delta + (n-1)\delta^2\right) \leq \left(\frac{\delta}{n} + \delta^2\right) \cdot T,\\
    &\sum_{t \in [T]}\E[f_t(\ba_t^{\OCRS})] \geq \sum_{I \in [T/n]}(\alpha\delta \cdot n) = \alpha\delta \cdot T.
\end{align*}
Hence,
\begin{align*}
    \cumreg &= \sum_{t \in [T]}\E[f_t(\ba_t^{\OCRS})] - \sum_{t \in [T]}\E[f_t(\ba_t^{\ALG})]\\
    &\geq \left(\alpha\delta - \frac{\delta}{n} - \delta^2\right) \cdot T \\
    &= \frac{(\alpha - 1/n)^2}{4} \cdot T \in \Omega(T).
\end{align*}
The last step follows from the fact that $1/n \in o(\alpha)$, and $\alpha \in o(T)$.
\end{proof}

\section{Omitted proofs from Section ~\ref{sec:bandit}}

\regFullInfo*

\begin{proof}
We assume to have access to a regret minimizer for the set $\cP^{\bd}_{\cF}$ guaranteeing an upper bound on the cumulative regret up to time $T$ of $\cumreg$. Then, 

\begin{align*}
    \E\mleft[\sum_{t=1}^T f(\ba_t,\bw_t)\mright]& \ge \alpha \sum_{t=1}^T f(\bx_t,\bw_t)\\
    & \ge \alpha\mleft(\max_{\bx\in\cP^{\bd}_{\cF}}\,\sum_{t=1}^T f(\bx,\bw_t)-\cumreg\mright)\\
    & = \alpha\mleft( \max_{\ba}\sum_{t=1}^Tf(\ba,\bw_t)-\cumreg\mright),
\end{align*}
where the first inequality follows from the fact that Algorithm \ref{alg:alg full info} employs a suitable temporal OCRS $\hat\pi$ to select $\ba_t$: for each $e\in \cE$, the probability with which the OCRS selects $e$ is at least $\alpha\cdot x_{t,e}$, and since $f$ is a linear mapping (in particular, it is defined as the scalar product between a vector of weights and the choice at $t$) the above inequality holds. The second inequality is by no-regret property of the regret minimizer for decision space $\cP^{\bd}_{\cF}$. This concludes the proof.
\end{proof}

\semiBandit*
\begin{proof}
We start by computing a lower bound on the average reward the algorithm gets. Algorithm \ref{alg:alg semi bandit} splits its decisions into $Z$ blocks, and, at each $\tau\in[Z]$, chooses the action $\bx_\tau$ suggested by the \textsc{RM}, unless the stage is one of the randomly sampled exploration steps. Then, we can write
\begin{align*}
    \frac{1}{T} \cdot \E\mleft[ \sum_{t=1}^T f(\ba_t,\bw_t)\mright] &\ge \frac{\alpha}{T} \cdot  \sum_{\tau\in[Z]}\sum_{t\in I_\tau} f(\bx_t,\bw_t) \\
    &\ge \frac{\alpha}{T}\sum_{\tau\in[Z]}\sum_{t\in I_\tau} f(\bx_\tau,\bw_t) - \alpha \frac{m^2 Z}{T} \\
    &= \frac{\alpha}{T} \cdot \sum_{\tau\in[Z]} \sum_{e\in E} x_{\tau,e} \sum_{t\in I_\tau} f(\bone_e,\bw_t) - \alpha \frac{m^2 Z}{T}\\
    &= \frac{\alpha}{Z} \cdot \sum_{\tau\in[Z]} \sum_{e\in E} x_{\tau,e} \cdot \E\mleft[\tilde f_\tau(e)\mright] - \alpha \frac{m^2 Z}{T},
\end{align*}
where the first inequality is by the use of a temporal OCRS to select $\ba_t$, and the second inequality is obtained by subtracting the \emph{worst-case} costs incurred during exploration; note that the $m^2$ factor in the second inequality is due to the fact that at each of the $m$ exploration stages, we can lose at most $m$. The last equality is by definition of the unbiased estimator, since the value of $f$ is observed $T/Z$ times (once for every block) in expectation.

We can now bound from below the rightmost expression we just obtained by using the guarantees of the regret-minimizer.
\begin{align*}
\frac{\alpha}{Z} \cdot \sum_{\tau\in[Z]} \sum_{e\in \cE} x_{\tau,e} \cdot \E\mleft[\tilde f_\tau(e)\mright] - \alpha \frac{m^2 Z}{T}&\ge 
    \frac{\alpha}{Z} \cdot \E\mleft[ \max_{\bx\in \cP_{\cF}^{\bd}}\sum_{\tau\in[Z]} \sum_{e\in \cE} x_e \tilde f_\tau(e) - \cumreg[Z]\mright] - \alpha \frac{m^2 Z}{T}\\
    &\hspace{-1cm}= \frac{\alpha}{Z} \max_{\bx\in \cP_{\cF}^{\bd}}\sum_{\tau\in[Z]} \sum_{e\in \cE} x_e \cdot \E\mleft[\tilde f_\tau(e) \mright] - \frac{\alpha}{Z}\cumreg[Z] - \alpha \frac{m^2 Z}{T}\\
    &\hspace{-1cm}= \frac{\alpha}{T} \max_{\ba\in \cF^{\bd}}\sum_{\tau\in[Z]} \sum_{e\in \cE} x_e \sum_{t\in I_\tau}f(\bone_e,\bw_t) - \frac{\alpha}{Z}\cumreg[Z] - \alpha \frac{m^2 Z}{T}\\
    &\hspace{-1cm}= \frac{\alpha}{T} \max_{\ba\in \cF^{\bd}}\sum_{t=1}^T f(\ba_t,\bw_t) - \frac{\alpha}{Z}\cumreg[Z] - \alpha \frac{m^2 Z}{T},
\end{align*}
where we used unbiasedness of $\tilde f_\tau(e)$, and the fact that the value of optimal fractional vector in the polytope is the same value provided by the best superarm (i.e., best vertex of the polytope) by convexity. The third equality follows from expanding the expectation of the unbiased estimator (i.e. $\E\mleft[\tilde f_\tau(e) \mright] \defeq \frac{Z}{T} \cdot \sum_{t\in I_\tau}f(\bone_e,\bw_t)$). Let us now rearrange the last expression and compute the cumulative regret: 
\begin{align*}
     \alpha \max_{\ba\in \cF^{\bd}}\sum_{t=1}^T f(\ba_t,\bw_t) - \E\mleft[ \sum_{t=1}^T f(\ba_t,\bw_t)\mright] &\leq \frac{\alpha}{Z}\underbrace{\cumreg[Z]}_{\leq \tilde O (\sqrt{Z})} T  + \alpha m^2 Z \\
     &\leq \tilde O(T^{2/3}),
\end{align*}
where in the last step we set $Z=T^{2/3}$ and obtain the desired upper bound on regret (the term $\alpha m^2$ is incorporated in the $\tilde O$ notation). The theorem follows.
\end{proof}

\section{Further Related Works}\label{app:D}
\paragraph{CRS and OCRS.} Contention resolution schemes (CRS) were introduced by~\citet{CVZ11} as a powerful rounding technique in the context of submodular maximization. The CRS framework was extended to online contention resolution schemes (OCRS) for online selection problems by \citet{FSZ16}, who provided
OCRSs for different problems, including intersections of matroids, matchings, and prophet inequalities. \citet{EFGT20} recently extended OCRS to batched arrivals, providing a constant competitive ratio for stochastic max-weight matching in vertex and edge arrival models.

\paragraph{Combinatorial Bandits.} The problem of combinatorial bandits was first studied in the context of \emph{online shortest paths}~\cite{awerbuch2008online,gyorgy2007line}, and the general version of the problem is due to \citet{cesa2012combinatorial}. Improved regret bounds can be achieved in the case of combinatorial bandits with semi-bandit feedback~ (see, \emph{e.g.}, \cite{chen2013combinatorial,kveton2015tight,audibert2014regret}). A related problem is that of \emph{linear bandits}~\cite{awerbuch2008online,mcmahan2004online}, which admit computationally efficient algorithms in the case in which the action set is convex~\cite{abernethy2009competing}.

\paragraph{Blocking bandits.} In \emph{blocking bandits} \cite{BSSS19} the arm that is played is blocked for a specific number of stages.  Blocking bandits have recently been studied in contextual \cite{BPCS21}, combinatorial \cite{AtsidakouP0CS21}, and adversarial \cite{BishopCMT20} settings. Our bandit model differs from blocking bandits since we consider each instance of the problem confined within each stage. In addition, the online full information problems that are solved in most blocking bandits papers \cite{AtsidakouP0CS21, BPCS21, dickerson2018allocation} only addresses specific cases of the fully dynamic online selection problem, which we solve in entire generality. 

\paragraph{Sleeping bandits.} As mentioned, our problem is similar to that of \emph{sleeping bandits} (see \cite{KNMS10} and follow-up papers), but at the same time the two models differ in a number of ways. Just like the sleeping bandits case, the adversary in our setting decides which actions we can perform by setting arbitrary activity times at each $t$. The crucial difference between the two settings is that, in sleeping bandits, once an adversary has chosen the available actions for a given stage, they have to communicate them all at once to the algorithm. In our case, instead, the adversary can choose the available actions within a given stage as the elements arrive, so it is, in some sense, ``more dynamic''. In particular, in the temporal setting there are two levels of adaptivity for the adversary: on one hand, the adversary may or may not be adaptive across stages (this is the classic bandit notion of adaptivity). On the other hand, the adversary may or may not be adaptive within the same stage (which is the notion of online algorithms adaptivity).

\end{document}